\documentclass[final]{colt2020} 

\allowdisplaybreaks
\usepackage{thmtools,thm-restate}
\SetKwInput{KwInput}{Input}                
\SetKwInput{KwOutput}{Output}              

\usepackage{xcolor}
\newcommand{\yingyu}[1]{}
\newcommand{\hui}[1]{\textcolor{green}{Revise: #1}}

\title[Learning Entangled Single-Sample Gaussians in the SoS Model]{Learning Entangled Single-Sample Gaussians in the Subset-of-Signals Model}
\usepackage{times}

\coltauthor{%
 \Name{Yingyu Liang} \Email{yliang@cs.wisc.edu}\\
 \addr University of Wisconsin-Madison
 \AND
 \Name{Hui Yuan\thanks{The work was done during the summer internship of H. Yuan at the University of Wisconsin-Madison.} 
 } \Email{yuanhui@mail.ustc.edu.cn}\\
 \addr University of Science and Technology of China
}

\begin{document}

\maketitle

\begin{abstract}%
In the setting of entangled single-sample distributions, the goal is to estimate some common parameter shared by a family of $n$ distributions, given one single sample from each distribution. This paper studies mean estimation for entangled single-sample Gaussians that have a common mean but different unknown variances. We propose the subset-of-signals model where an unknown subset of $m$ variances are bounded by 1 while there are no assumptions on the other variances. In this model, we analyze a simple and natural method based on iteratively averaging the truncated samples, and show that the method achieves error $O \left(\frac{\sqrt{n\ln n}}{m}\right)$ with high probability when $m=\Omega(\sqrt{n\ln n})$, matching existing bounds for this range of $m$. We further prove lower bounds, showing that the error is $\Omega\left(\left(\frac{n}{m^4}\right)^{1/2}\right)$ when $m$ is between $\Omega(\ln n)$ and $O(n^{1/4})$, and the error is $\Omega\left(\left(\frac{n}{m^4}\right)^{1/6}\right)$ when $m$ is between $\Omega(n^{1/4})$ and $O(n^{1 - \epsilon})$ for an arbitrarily small $\epsilon>0$, improving existing lower bounds and extending to a wider range of $m$.
\end{abstract}


\section{Introduction} \label{sec:intro}

This work considers the novel parameter estimation setting called entangled single-sample distributions. 
In this setting, distributions are entangled in the sense that they share some common parameter and our goal is to estimate the common parameter based on one sample from each distributions obtained. 
We focus on the mean estimation problem in the subset-of-signals model when the distributions are Gaussians. 
In this problem, we have $n$ independent Gaussians with a common mean with different unknown variances. Given one sample from each of the Gaussians, our goal is to estimate the mean.

There can be different configurations of the unknown variances. In this work, we propose a basic model called subset-of-signals, which assumes that an unknown subset of $m$ variances are bounded by 1 while there are no assumptions on the other variances. Equivalently, $\sigma_{(m)} \le 1$ where $\sigma_{(m)}$ is the $m$-th smallest value in $\{\sigma_i\}_{i=1}^n$.
The subset-of-signals model gives a simple setting specifying the possible configurations of $n$ unknown variances $\left\{\sigma_i\right\}_{i=1}^{n}$ for analysis. While even in this simple setting, the optimal rates of mean estimation for entangled single-sample Gaussians are still unknown (for most values of $m$). 

The setting of entangled single-sample distributions is motivated for both theoretical and practical reasons.
From the theoretical perspective, it goes beyond the typical i.i.d.\ setting and raises many interesting open questions in the most fundamental topics like mean estimation of Gaussians. 
It can also be viewed as a generalization of the traditional mixture modeling, since the number of distinct mixture components could grow with the number of samples and even be as large as the number of samples. From the practical perspective, traditional i.i.d.\ assumption can lead to a bad modeling of data in modern applications, where various forms of heterogeneity occur.
In particular, entangled Gaussians capture heteroscedastic noises in various applications and thus can be a natural model for studying robustness.

Though theoretically interesting and practically important, few studies exist in this setting. \cite{chierichetti2014learning} considered the mean estimation for entangled Gaussians and showed the existence of a gap between estimation error rates of the best possible estimator in this setting and the maximum likelihood estimator when the variances are known. It focused on the case where most samples are ``high-noised" (i.e., most variances are large), and provided bounds in terms of $\sigma_{(m)}$ with small $m$ like $\Theta(\ln n)$.
\cite{pensia2019estimating} considered means estimation for symmetric, unimodal distributions with sharpened bounds, 
and provided extensive discussion on the performance of their estimators in different configurations of the variances. Many questions are still largely open. In particular, when instantiated in the subset-of-signals model, existing studies provide interesting upper bounds and lower bounds but a large gap remains. See the related work section and remarks after our theorems for more details. 

This work thus proposes the subset-of-signals model and attempts to gain better understanding on the problem. For the upper bound, we aim to achieve a vanishing error bound (i.e., the error bound tends to 0 when $n\rightarrow +\infty$). We analyze a simple algorithm based on iteratively averaging the truncated samples: it keeps an iterate and each time it truncates the samples in an interval around the current iterate and then averages the truncated samples to compute the next iterate. We also prove lower bounds for a wide range of $m$, improving known bounds. Our main results are summarized below. 


\subsection{Main Results}

\paragraph{Problem Setup.}
Suppose we have $n$ independent samples $x_i \sim \mathcal{N}(\mu^\star, \sigma_i^2)$, where the distributions have a common mean $\mu^\star$ but different variances $\sigma_i^2$. The mean and variances are all unknown. We consider the subset-of-signal model, where an unknown subset of $m$ variances are bounded by 1 while there are no assumptions on the other variances. That is, $\sigma_{(m)} \le 1$ where $\sigma_{(m)}$ is the $m$-th smallest value in $\{\sigma_i\}_{i=1}^n$. Our goal is to estimate the common mean $\mu^{\star}$ from the samples $\{x_i\}_{i=1}^n$. 

As usual, we use $f(n,m) = O(g(n,m))$ (or $f(n,m) \lesssim g(n,m)$) if there exist $N, M$ and $C>0$ such that when $n> N$ and $m>M$, $f(n,m) \le C g(n,m)$. $f = \tilde{O}(g)$ hides logarithmic terms. $f = \Omega(g)$ (or $f \gtrsim g$), $f=\Theta(g)$ (or $f \simeq g$), $f=o(g)$, and $f=\omega(g)$ are defined as usual.

\paragraph{Upper bound.}
We obtain the following result for an algorithm based on iteratively averaging  truncated samples(see Algorithm~\ref{alg:iter_trun} for the details).

\begin{restatable}{theorem}{upperbound}
\label{thm:upperbound}
If $\sigma_{(m)} \le 1$ for $m = \Omega(\sqrt{n\ln n})$, then with probability at least $1-1/n$, the output $\hat\mu$ of Algorithm~\ref{alg:iter_trun} satisfies 
\begin{align*}
    |\hat\mu - \mu^\star| \lesssim \frac{\sqrt{n\ln n}}{m}.
\end{align*}
\end{restatable}

The result shows that the algorithm can achieve a vanishing error when $m = \omega(\sqrt{n\ln n})$. Therefore, we can achieve vanishing error with only an $\omega(\sqrt{\ln n/n})$ fraction of samples with bounded variances. This means even when the noisy samples dominates the data and the fraction of signals diminishes when $n \rightarrow +\infty$, we can still obtain accurate estimation. 
The result also shows that when there are only a constant fraction of ``heavy-noised'' data (i.e., $m = \Theta(n)$), the error rate is $O(\sqrt{\ln n /n})$, which matches the optimal error rate $O(1/\sqrt{n})$ up to a logarithmic factor. Our result matches the best bound known: the hybrid estimator proposed in \cite{pensia2019estimating} achieved $O(\sqrt{n}\ln n/m)$ in the subset-of-signals model but for essentially all values of $m$ (Theorem 6 in their paper). (One should be able to tighten their analysis to get $O(\sqrt{n\ln n}/m)$ with high probability.) Furthermore, median estimators can already achieve such a bound for the range $m = \Omega(\sqrt{n\ln n})$ (e.g., Lemma 5 in their paper).
Our contribution is to show that iterative truncation can also achieve such a guarantee. The iterative truncation is natural and widely used in practice, so our analysis can be viewed as a justification for this heuristic. 

Our upper bound is in sharp contrast to the robust mean estimation in the commonly studied adversarial contamination model~\citep{valiant1985learning,huber2011robust,diakonikolas2019robust}, where an $\epsilon$ fraction of the data are adversarially modified and it has been shown that vanishing error is impossible when $\epsilon = \Omega(1)$. This means that the entangled distributions setting can be much more benign than the adversarial contamination model. For mean estimation for entangled Gaussians in the subset-of-signals model, one can view it as an adversary picking $n-m$ variances but having no control over the sampling process after specifying those variances. That is, it is a semi-adversarial model and can be much more benign than the fully adversarial contamination model.


\begin{figure}
    \centering
    \includegraphics[width=0.75\textwidth]{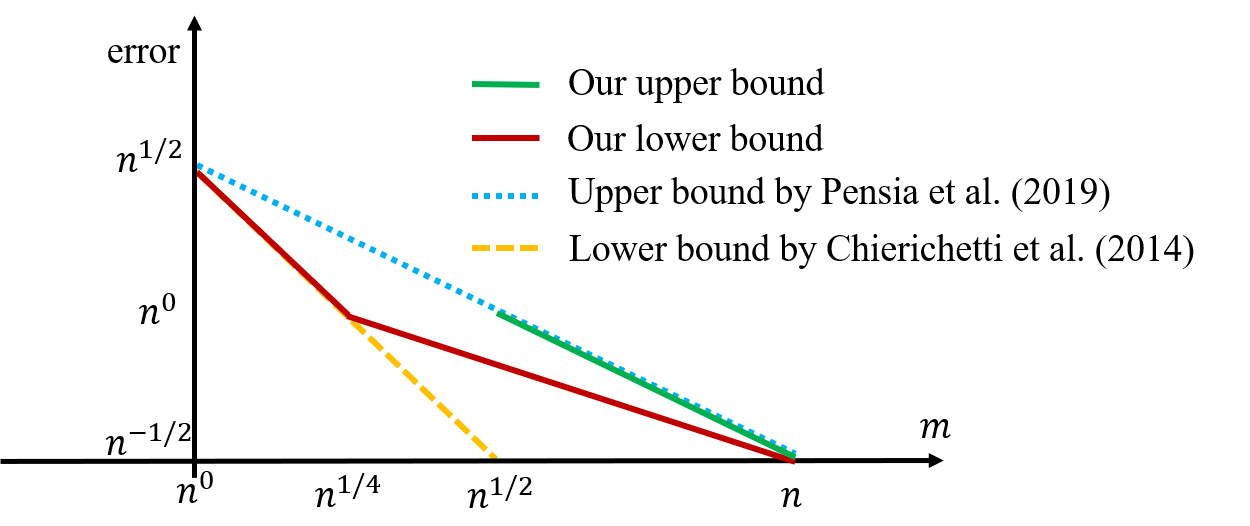}
    \caption{Our bounds and those from previous works for mean estimation of entangled Gaussians in the subset-of-signals model. $x$-axis is the number of Gaussians with variances $1$, $y$-axis is the error. See the text for the details of the bounds.}
    \label{fig:bounds}
\end{figure}

\paragraph{Lower bound.}

We now turn to the lower bound. Note that an instance of our problem is specified by $\mu^\star$ and $\{\sigma_i\}_{i=1}^n$.


\begin{restatable}{theorem}{lowerbound} \label{thm:lowerbound}
Suppose $\sigma_{(m)} \le 1$.
\begin{itemize}
    \item If $m=\Omega(\ln n)$ and $m=O(n^{1/4})$, then there exist a family of instances and a distribution over these instances such that any estimator has expected error $\Omega\left(\left(\frac{n}{m^4}\right)^{1/2}\right)$.
    
    \item For any arbitrarily small $\epsilon>0$, if $m$ is between $\Omega(n^{1/4})$ and $O(n^{1 - \epsilon})$, then there exist a family of instances and a distribution over these instances such that any estimator has expected error
    $\Omega\left(\left(\frac{n}{m^4}\right)^{1/6}\right)$. 
    
\end{itemize}
\end{restatable}

The bound is for a distribution over the instances, which then implies the typical minimax bound.
The result shows that when $m=O(n^{1/4})$, it is impossible to obtain vanishing error. When $m$ is as small as $\Theta(\ln n)$, the error is $\tilde{\Omega}(\sqrt{n})$, paying a factor of $\tilde{\Omega}(\sqrt{n})$ compared to the oracle bound $O(1/\sqrt{m})$ when the $m$ bounded variance samples are known.  
When $m=\Omega(n^{1/4})$, the lower bound does not exclude the possibility of vanishing error. On the other hand, it shows that one needs to pay a factor of $\Omega\left(\left(\frac{n}{m}\right)^{1/6}\right)$, compared to the oracle bound $O(1/\sqrt{m})$ when the $m$ bounded variance samples are known. It also shows that one needs to pay a factor of $\Omega\left(\left(\frac{n}{m}\right)^{2/3}\right)$, compared to the bound $O(1/\sqrt{n})$ when all samples have bounded variance $1$. 

Our result extends and improves the lower bound in \cite{chierichetti2014learning}.
Their bound is $\Omega\left(\left(\frac{n}{m^4}\right)^{1/2}\right)$ for $m$ between $\Omega(\ln n)$ and $o(\sqrt{n})$. Our result extends the range of $m$ by including the values between $\Omega(n^{1/2})$ and $O(n^{1 - \epsilon})$ (for any arbitrarily small $\epsilon>0$). It also improves their bound in the range between $\Omega(n^{1/4})$ and $o(n^{1/2})$, by a factor of $\Omega\left(\left(\frac{m^4}{n}\right)^{1/3}\right)$. 

Figure~\ref{fig:bounds} provides an illustration summarizing the known upper and lower bounds for mean estimation of entangled single-sample Gaussians in the subset-of-signals model.
There is still a gap between the known upper and lower bounds. A natural direction is to close the gap and obtain the optimal rates, which we left as future work. 

\section{Related Work} \label{sec:related}

\paragraph{Entangled distributions.}
This setting is first studied by~\citet{chierichetti2014learning}, which considered mean estimation for entangled Gaussians and presented a algorithm combining the $k$-median and the $k$-shortest gap algorithms. It also showed the existence of a gap between the error rates of the best possible estimator in this setting and the maximum likelihood estimator when the variances are known. \citet{pensia2019estimating} considered a more general class of distributions (unimodal and symmetric) and provided analysis on both individual estimator ($r$-modal interval, $k$-shortest gap, $k$-median estimators) and hybrid estimator, which combines Median estimator with Shortest Gap or Modal Interval estimator. They also discussed slight relaxation of the symmetry assumption and provided extensions to linear regression. 
Our work focuses on the subset-of-signals model that allows to study the minimax rate and helps a clearer understanding of the problem (but our results can also be used for some other configurations). The algorithm we analyzed is based on the natural iterative truncation heuristics frequently used in practice to handle heteroscedastic noises, and our bound for it matches the best known rates (obtained by the hybrid estimator in \citet{pensia2019estimating}) in the range $m=\Omega(\sqrt{n\ln n})$. We also extends (to a wider range of $m$) and improves the lower bound in~ \citet{chierichetti2014learning}.

\citet{yuan2020learning} considered mean estimation for entangled distributions, but the distributions are not assumed to be Gaussians (it only assumed the distributions have the same mean and their variances exist). Due to this generality, their upper bound is significantly worse than ours: it's only for $m \ge 4n/5$ (i.e., only a constant fraction of high noise points); it does not achieve a vanishing error when $n$ tends to infinity. The paper doesn't provide lower bounds.
Their algorithm is also based on iterative truncation, but has the following important difference: it removes a fixed fraction of data points in each iteration, rather than doing adaptive truncation. In contrast, our algorithm uses adaptive truncation interval lengths. This is crucial to obtain our results, since intuitively the best bias-variance trade-off introduced by the truncation can only be achieved with adaptive truncation.

The entangled distributions setting is also closely related to robust estimation, which have been extensively studied in the literature of both classic statistics and machine learning theory. 

\paragraph{Robust mean estimation.}
There are several classes of data distribution models for robust mean estimators. The most commonly addressed is adversarial contamination model, whose origin can be traced back to the malicious noise model by~\citet{valiant1985learning} and the contamination model by~\citet{huber2011robust}. Under contamination, mean estimation has been investigated in~\cite{diakonikolas2017being,diakonikolas2019robust,cheng2019high}. 
Another related model is the mixture of distributions. There has been steady progress in algorithms for leaning mixtures, in particular, leaning Gaussian mixtures. Starting from~\citet{dasgupta1999learning}, a rich collection of results are provided in many studies, such as~\cite{sanjeev2001learning, achlioptas2005spectral, kannan2005spectral,belkin2010polynomial, belkin2010toward, kalai2010efficiently,moitra2010settling,diakonikolas2018robustly}.

\paragraph{Heteroscedastic models.} The setting of entangled distributions is also closely related to heteroscedastic models, which have been a classic topic in statistics. For example, in heterogeneous linear regression~\citep{munoz1986regression,vicari2013multivariate}, the errors for different response variables may have different variances, and weighted least squares has been used for estimating the parameters in this setting. Another example is Principal Component Analysis for heteroscedastic data~\citep{hong2018asymptotic,hong2018optimally,zhang2018heteroskedastic}. The entangled Gaussians can be viewed as a model of mean estimation in the presence of heteroscedastic noises.

\section{Upper Bound} \label{sec:upper}

The na\"ive method of averaging all samples cannot achieve a small error when some distributions have large variances. A natural idea is then to reduce the variances. Truncation is a frequently used heuristic, i.e., projecting the samples to an interval (around a current estimation) to get controlled variances. However, while averaging the original samples is consistent, truncation can lead to bias. So truncation introduces some form of bias-variance tradeoff and the width of the interval controls the tradeoff. Intuitively, the best width will depend on how aligned the interval is with the true mean; for intervals around estimations of different error, the width for the best tradeoff can be different. Therefore, we consider iterative truncation using adaptive widths for the interval. 
 
\begin{algorithm2e}[t]\label{alg:iter_trun}
\caption{Mean Estimation via Iterative Truncation}
\SetAlgoLined
\KwInput{$\{x_i\}_{i=1}^n$, initialization $\mu_0$, and parameters $B, m$ s.t.\ $B\ge 2|\mu_0 - \mu^*|, \sigma_{(m)} \le 1$}
 Set $\delta^{(0)} = B, \mu_0^{(0)} = \mu_0, K= \lfloor\log_2 \delta^{(0)} \rfloor, T=\lceil 64n\ln n / m \rceil$
 
 \For{$k=0, 1, \ldots, K$}{
    \For{$t=0, 1, \ldots, T$}{
    $\Delta^{(k)}_t = [\mu^{(k)}_t - \delta^{(k)}, \mu^{(k)}_t + \delta^{(k)}]$ 
    
    $\mu^{(k)}_{t+1}  = \frac{1}{n} \sum_{i=1}^n \phi(x_i; \Delta^{(k)}_t)$ 
    \quad\quad\tcp{$\phi$ is defined in Eqn (\ref{eq:phi})}
    }
    $\mu^{(k+1)}_0 = \mu^{(k)}_{T+1}$, $\delta^{(k+1)} = \delta^{(k)}/2$
 }
\KwOutput{$\hat{\mu} \leftarrow \mu^{(K)}_{T+1}$}
\end{algorithm2e}

Algorithm~\ref{alg:iter_trun} describes the details of our method. 
Given an initial estimation $\mu_0$, it averages the truncated data in an interval around the estimation iteratively. 
In particular, the algorithm has $K$ stages, and each stage has $T$ steps. In step $t$ of stage $k$, given a current estimation $\mu_t^{(k)}$ and a width parameter $\delta_t^{(k)}$, the algorithm computes the new estimation $\mu_{t+1}^{(k)}$ by averaging the truncated data $\phi(x_i; \Delta_t^{(k)})$, where $\Delta_t^{(k)}$ is the interval  around $\mu_t^{(k)}$ with radius $\delta_t^{(k)}$, and $\phi$ is defined as:
\begin{align} \label{eq:phi}
  \phi(x; [a, b]) & = 
  \begin{cases}
  a, & \quad \textrm{if } x < a,
  \\
  x, & \quad \textrm{if }  a \le x \le b,
  \\
  b, & \quad \textrm{if }  x > b.
  \end{cases}
\end{align}

For this algorithm, we prove the following guarantee.
\upperbound*

\paragraph{Remark.}
The algorithm needs an initialization $\mu_0$ and parameter $B$. There exist methods to achieve this, e.g., set $\mu_0$ as the sample mean and $B$ as two times the diameter of the sample points.

\paragraph{Remark.}
Our proof actually gives more general results.
Let $m({\delta}) = \max\{i: \sigma_{(i)} \le \delta\}$ and let $H^\sigma_{\delta}$ be the harmonic mean of $\{\max(\sigma_i, \delta)\}_{i=1}^n$, i.e.,
$ H^\sigma_{\delta} = n/(\sum_{i=1}^n \frac{1}{\max(\sigma_i, \delta)})$.
Then our analysis shows that for any $k$ in the algorithm, the estimation at the end of the $k$-th iteration satisfies $|\mu^{(k)}_{T+1} - \mu^\star| \lesssim H^\sigma_{\delta^{(k)}} \sqrt{\frac{ \ln n}{ n} }$. That is, with probability at least $1- 1/n$, the algorithm can output an estimation $\hat\mu$ (by setting proper $K$ and $T$) for any $\delta$ with $m(\delta) = \Omega(\sqrt{n \ln n})$, such that
\begin{align} \label{eqn:general_upper_1}
    |\hat\mu - \mu^\star| \lesssim H^\sigma_{\delta} \sqrt{\frac{ \ln n}{ n} }.
\end{align}
Since $H^\sigma_{\delta}   \le  n \delta/m(\delta)$, the error is $ \lesssim \frac{ \delta \sqrt{n \ln n} }{ m(\delta)}$. So for any $t \ge m =\Omega(\sqrt{n \ln n})$, by setting $\delta=\sigma_{(t)}$ (the $t$-th smallest variance), we can get with probability $1- 1/n$,
\begin{align} \label{eqn:general_upper_2}
    |\hat\mu - \mu^\star| \lesssim \frac{ \sigma_{(t)} \sqrt{n \ln n} }{t}
\end{align}
When $t=m$, we recover the bound in the theorem. 

The more general results are more adaptive. First, they can be applied to more general threshold values $\delta$. For example, for the configuration of variances where $\sigma_{(t)}$ can increase with $n$, one can still get vanishing error when $\sigma_{(t)} = o({t/  \sqrt{n \ln n} })$. Second,  (\ref{eqn:general_upper_1}) can be applied to different configurations of $\sigma_i$'s and obtain better bounds. When $\sigma_{(i)}$'s for $i > m$ are benign, (\ref{eqn:general_upper_1}) shows that they can help the estimation and quantifies the provided information with the notion $H^\sigma_{\delta}$. 

\paragraph{Remark.}
We would also like to point out, the hybrid estimator proposed in \cite{pensia2019estimating} also achieved almost the same upper bound $O(\sqrt{n}\ln n/m)$ as ours in the subset-of-signals model, but for essentially all values of $m$. (Their analysis can be tightened to get $O(\sqrt{n\ln n}/m)$).
Their bound is obtained by combining two estimators, and depends on a notion $r_k$, the length of the smallest interval containing $k$ samples. Furthermore, the $k$-median estimator (with proper $k$) can also achieve the bound for the range $m=\Omega(\sqrt{n\ln n})$. In comparison, our bound is for the iterative truncation heuristic frequently used in practice, and depends on the notion $H^\sigma_{\delta}$. 

More details of the existing bounds are as follows.
\cite{chierichetti2014learning} achieved an error bound $\min_{2 \leq k\leq \log n}\tilde{O}\left(n^{1/2(1+1/(k-1))}\sigma_k\right)$. Among all estimators studied in~\cite{pensia2019estimating}, the superior performance is obtained by the hybrid estimators, which includes version (1): combining $k_1$-median with $k_2$-shorth and version (2): combining $k_1$-median with modal interval estimator. These two versions achieve similar guarantees. Version $1$ of the hybrid estimator outputs $\hat{\mu}_{k_1, k_2}$ such that $\left|\hat{\mu}_{k_1, k_2}-\mu\right| \leq \frac{4\sqrt{n}\log n}{k_2} r_{2k_2}$ with probability $1-2 \exp(-c^{\prime}k_2)-2 \exp(-c \log^2 n)$, where $k_1=\sqrt{n}\log n$ and $k_2 \geq C \log n$. Here $r_{k}$ is defined as $\inf\left\{ r: \frac{1}{n} \sum_{i=1}^n \mathbb{P}(|x_i-\mu^{\star}| \leq r)\geq \frac{k}{n} \right\}$. So the error bound varies with specific configurations of the variances. Furthermore, the modal interval estimator or the shorth estimator still work for small $m$'s, so their bound holds also for $m = \tilde{O}(n^{1/2})$.

\subsection{Proof of Theorem 1}
To prove the theorem, we first focus on one stage and omit the superscript $(k)$. Define 
\begin{align}
e_t & := |\mu_t - \mu^{\star}|,
\\
z_i & := \phi(x_i;\Delta_t) - \mu^\star,
\\
\bar{z}_i &: = z_i - \mathbb{E} z_i.
\end{align}
We have
\begin{align}
     \mu_{t+1} - \mu^{\star} = \frac{ \sum_{i=1}^n (\phi(x_i;\Delta_t) - \mu^\star) }{n} = \frac{1}{n} \sum_{i=1}^n z_i.
\end{align}
To bound $|\sum_{i=1}^n z_i|$, we need to bound $\bar{z}_i$'s and $|\mathbb{E} z_i|$'s. 
Since $z_i$ is 1-Lipschitz w.r.t.\ $\mu_t$, a standard $\epsilon$-net argument gives a uniform concentration bound of $\bar{z}_i$'s in Lemma~\ref{lem:uniform}. 
$|\mathbb{E} z_i|$ is bounded in Lemma~\ref{lem:bias}. See Appendix~\ref{app:upperbound} for their proofs.

\begin{lemma} \label{lem:uniform}
Let $ z_i(\mu) = \phi(x_i;[\mu-\delta, \mu+\delta]) - \mu^\star,
\bar{z}_i(\mu) = z_i(\mu) - \mathbb{E} z_i(\mu)
$.
With probability at least $1-1/n^3$, for any $\mu$ satisfying $|\mu - \mu^\star| \le \delta_e$,
we have
\begin{align*}
    \left|\sum_{i=1}^n \bar{z}_i(\mu) \right| \lesssim \delta \sqrt{n \ln n} + \frac{\delta_e}{n}. 
\end{align*} 
\end{lemma}

\begin{lemma} \label{lem:bias}
Let $z_i(\mu) = \phi(x_i;[\mu-\delta, \mu+\delta]) - \mu^\star$ and $\delta_e = |\mu - \mu^\star| $. Then
\begin{align*}
  |\mathbb{E} z_i| 
  & \le \delta_e \left(1 - \frac{1}{5}\frac{\delta}{ \max\{\delta_e,\delta\}} \frac{\delta}{\max\{\sigma_i,\delta\}} \right).
\end{align*} 
\end{lemma}

Using these two lemmas, we can analyze one iteration of the algorithm.

\begin{lemma} \label{lem:one_iter}
If $\delta \ge e_t$, then with probability at least $1 - 1/n^3$, 
\begin{align*}
    e_{t+1} \le C \delta\sqrt{\frac{ \ln n}{n}} + e_t \left(1- \frac{\delta}{5 H^\sigma_{\delta}} \right)
\end{align*}
where $H^\sigma_{\delta}$ is the harmonic mean of $\{\max(\sigma_i, \delta)\}_{i=1}^n$:$
    H^\sigma_{\delta} = n/\sum_{i=1}^n \frac{1}{\max(\sigma_i, \delta)}.
$
\end{lemma}

\begin{proof}
By Lemma~\ref{lem:uniform} and Lemma~\ref{lem:bias}, with probability at least $1 - 1/n^3$,  
\begin{align*} 
  \left|\sum_{i=1}^n z_i \right| 
  & \le \left|\sum_{i=1}^n \bar{z}_i \right| + \sum_{i=1}^n |\mathbb{E} z_i| 
  \\
  & \le C \delta \sqrt{n \ln n } + \frac{e_t}{n} + e_t \left( n - \frac{\delta }{5}\sum_{i=1}^n \frac{1}{\max(\sigma_i, \delta ) } \right).
\end{align*}
This leads to the final bound.
\end{proof}

Now we are ready to prove Theorem~\ref{thm:upperbound}.

At stage $k=0$, we have $\delta^{(k)} \ge 2 |\mu^{(k)}_0 - \mu^\star|$. Suppose this is true for stage $k < K$, we show that it is true for $k+1$.

In stage $k$, we have $\delta^{(k)} \ge 2 e_t$ for $t = 0$. Suppose this is true for a step $t \le T$, we show that it is true for $t+1$. Let $m({\delta}) = \max\{i: \sigma_{(i)} \le \delta\}$. We have 
\begin{align*}
H^\sigma_{\delta} = \frac{n}{\sum_{i=1}^n \frac{1}{\max(\sigma_i, \delta)} } \le \frac{n}{ m(\delta) \frac{1}{\delta} }  = \frac{ n \delta}{m(\delta)}. 
\end{align*}
Then by Lemma~\ref{lem:one_iter},
\begin{align*}
    e_{t+1} \le C \delta^{(k)} \sqrt{\frac{ \ln n}{n}} + e_t \left(1- \frac{m(\delta^{(k)})}{5 n} \right).
\end{align*}
If $e_t \gtrsim \frac{\delta^{(k)}\sqrt{n \ln n}}{m(\delta^{(k)})}$, 
$e_{t+1} \leq e_t \le \delta^{(k)}/2$. 
If $e_t \lesssim \frac{\delta^{(k)}\sqrt{n \ln n}}{m(\delta^{(k)})}$, we have $e_{t+1} \leq e_t + C \delta^{(k)} \sqrt{\frac{ \ln n}{n}} \lesssim \delta^{(k)} \frac{\sqrt{n \ln n}}{m(\delta^{(k)})} + \delta^{(k)} \sqrt{\frac{ \ln n}{n}}  \lesssim \delta^{(k)}\frac{\sqrt{n \ln n}}{m(\delta^{(k)})} \le \delta^{(k)}/4$. Therefore, we can always guarantee $e_{t} \le \delta^{(k)}/2$ for $t \le T$. Then Lemma~\ref{lem:one_iter} can be applied for all $t \le T$, and thus after $T$ iterations,
\begin{align*}
    e_T \le C \delta^{(k)} \sqrt{\frac{ \ln n}{n}} \sum_{i=0}^{T-1} \left(1- \frac{m(\delta^{(k)})}{5 n} \right)^i + \left(1- \frac{m(\delta^{(k)})}{5 n} \right)^T e_0 \lesssim \frac{ \delta^{(k)} \sqrt{n \ln n} }{ m(\delta^{(k)})}. 
\end{align*}
Since $m(\delta^{(k)}) \ge m$, $e_T < \delta^{(k)}/4$, so $\delta^{(k+1)} = \delta^{(k)}/2 > 2 e_t =  2 |\mu^{(k+1)}_0 - \mu^\star|$. 

Therefore,  $\delta^{(k)} \ge 2 |\mu^{(k)}_0 - \mu^\star|$ for all $k\le K$. Since $1 \le \delta^{(K)}$, at the end of stage $K$:
\begin{align*}
    e_T \lesssim \frac{ \delta^{(K)} \sqrt{n \ln n} }{ m(\delta^{(K)})} \lesssim \frac{  \sqrt{n \ln n} }{ m}. 
\end{align*}
This is $|\hat{\mu} - \mu^\star| \lesssim \frac{  \sqrt{n \ln n} }{ m}$.

\section{Lower Bound} \label{sec:lowerbound}

To complement the upper bound, we also provide the following lower bound. 

\lowerbound*

\paragraph{Remark.}  
The lower bound considers two ranges of $m$.
In the first range, the bound is $\tilde\Omega(\sqrt{n})$ at one end point $m=\Theta(\ln n)$, and is $\Omega(1)$ at the other end point $m=\Theta(n^{1/4})$. It decreases at a rate of $1/m^2$ as $m$ increases in this range. In the second range, the bound is $\Omega(1)$ at one end point $m=\Theta(n^{1/4})$, and is $\Omega(1/n^{1/2-2\epsilon/3})$ at the other end point $m=\Theta(n^{1-\epsilon})$ (for any arbitrarily small $\epsilon>0$). It decreases at a rate of $1/m^{2/3}$ as $m$ increases, which is slower than that in the first range. Roughly speaking, the bound excludes the possibility of vanishing error in the first range while still allows that in the second range, and the transition point is $m=\Theta(n^{1/4})$.

Our result extends and improves the lower bound in \cite{chierichetti2014learning}.
Their bound is $\Omega\left(\left(\frac{n}{m^4}\right)^{1/2}\right)$ for $m$ between $\Omega(\ln n)$ and $o(\sqrt{n})$. Our result extends the range of $m$ by including the values between $\Omega(n^{1/2})$ and $O(n^{1 - \epsilon})$ (for any arbitrarily small $\epsilon>0$). It also improves their bound in the range between $\Omega(n^{1/4})$ and $o(n^{1/2})$, by a factor of $\Omega\left(\left(\frac{m^4}{n}\right)^{1/3}\right)$. The improvement is obtained by a tighten analysis in the second range of $m$, which is discussed below.

\subsection{Proof of Theorem~\ref{thm:lowerbound}}


Our proof follows the high-level idea of \cite{chierichetti2014learning} but with a tightened analysis. We also consider the following distribution over a family of instances: $\sigma_i$'s are i.i.d.\ sampled; with probability $p$, $\sigma_i = \sigma_p$, and with probability $q = 1-p$, $\sigma_i = \sigma_q$;  $\mu^\star$ is uniform over $\{+L, -L\}$. Here, $p=m/n, \sigma_p = 1$, while $\sigma_q, L$ are parameters to be chosen.

The goal is then to choose $\sigma_q, L$ (based on $n,m$), such that conditioned on $\mu^\star = +L$ or $\mu^\star = -L$, the other choice of mean has a higher likelihood with a constant probability.
If this is true, then any estimator has an expected error $\Omega(L)$ over the above distribution on the instances and the randomness of the sample points.
When $m$ large enough, the probability that $\sigma_{(m/2)}>1$ is exponentially small. Then on the  distribution over the instances conditioned on $\sigma_{(m/2)}\le 1$, the lower bound holds under the assumption $\sigma_{(m/2)}\le 1$. By changing the variable $m$ to $2m$, the theorem follows.

We improve over \cite{chierichetti2014learning} by noting that, roughly speaking, the requirement on $\sigma_q$ when $m = \Omega(n^{1/4})$ is more relaxed compared to that when $m = O(n^{1/4})$. This allows us to set $\sigma_q$ differently to get improved results and also over a more general range of $m$, as detailed below.

Following the idea above, denote the likelihood of the mean being $L$ as $\mathcal{L}_+$, and the likelihood of the mean being $-L$ as $\mathcal{L}_-$. We will show that the log-likelihood ratio has sufficiently large variances so can be negative or positive with constant probabilities. From now on, we condition on the true mean is $L$ (the proof for the case with $-L$ is symmetric). Let $S_p = \{i: \sigma_i = \sigma_p\}$ and $S_q = \{i: \sigma_i = \sigma_q\}$. 
Define 
\begin{align}
    \label{equ:ratio_L}
    N_i & = \frac{p/\sigma_p}{q/\sigma_q} \exp\left(-\frac{(x_i-L)^2}{2} \left(\frac{1}{\sigma^2_p} - \frac{1}{\sigma^2_q} \right) \right)
    \\
    \label{equ:ratio_-L}
    D_i & = \frac{p/\sigma_p}{q/\sigma_q} \exp\left(-\frac{(x_i+L)^2}{2} \left(\frac{1}{\sigma^2_p} - \frac{1}{\sigma^2_q} \right)  \right).
\end{align}
Then we have
\begin{align*}
    \ln \frac{\mathcal{L}_+}{\mathcal{L}_-} 
    = & \sum_{i=1}^{n} \left( \ln \frac{1 + N_i}{1 + D_i} + \frac{2L}{\sigma_q^2} x_i  \right)
    =  \underbrace{
    \sum_{i\in S_p} \left( \ln \frac{1 + N_i}{1 + D_i} 
    + \frac{2L}{\sigma_q^2} x_i \right) 
    }_{X_p} + 
    \underbrace{
    \sum_{i\in S_q} \left( \ln \frac{1 + N_i}{1 + D_i} + \frac{2L}{\sigma_q^2} x_i \right)
    }_{X_q}.
\end{align*}
We next bound $X_q$ and $X_p$ respectively. 
The road map is to show that $X_q$ has sufficiently large variances so can make the log-likelihood ratio negative with constant probability, shown via the Berry-Essen Theorem. This requires computing the moments, so we first approximate $\ln \frac{1 + N_i}{1 + D_i}$ via the Taylor expansion of the function $\ln(1+x)$, and then compute the moments of the approximation. When $m = \Omega(n^{1/4})$, the likelihood of $x_i \in S_q$ is comparable to that of $x_i \in S_q$, so their ratio (as in \eqref{equ:ratio_L} or \eqref{equ:ratio_-L}) is in the same order as a constant. We thus use a tighter approximation for $\ln(1+N_i)$ and $\ln(1+D_i)$ in the log-likelihood ratio, and improve over \cite{chierichetti2014learning}.

\begin{lemma}\label{cl:sqt}
Suppose the mean is $L$, and $q > C_q p $, $\sigma_q  > C_\sigma \sigma_p$, $L < c_L \sigma_q $ for sufficiently large absolute constants $C_q, C_\sigma$ and a sufficiently small absolute constant $c_L$.%
\footnote{ 
$C_q$ is a constant chosen for the inequality $q > C_q p $. It doesn't depend on the value of $q$. Similar for $C_\sigma, c_L$ etc. 
}
Suppose $\frac{p/\sigma_p}{q/\sigma_q} < c_\alpha$ for a sufficiently small absolute constant $c_\alpha<1$. 
Let $t$ be a positive integer. 
Let $V_i = \sum_{j=1}^{2t-1} (-1)^{j+1} (N_i^j - D_i^j)/j$ and $Y_i = \frac{2L}{\sigma_q^2} x_i + V_i$. 
Then for $i \in S_q$,
\begin{align*}
    \mathbb{E}[Y_i] & \lesssim  \frac{L^2}{\sigma_q^2},
    \quad
    \mathbb{E}[Y_i^2]  \simeq \frac{p^2/\sigma_p}{q^2/\sigma_q} \min\left\{1, \frac{L^2}{\sigma_p^2}\right\} + \frac{L^2}{\sigma_q^2}. 
\end{align*}
And with probability at least $1 - n^{-\Theta(1)} - \exp\left(\Theta\left( U_q\right)\right)$,
$ 
    \left|X_q - \sum_{i\in S_q} Y_i \right| \lesssim U_q :=  \left(\frac{p/\sigma_p}{q/\sigma_q}\right)^{2t}\frac{\sigma_p}{\sigma_q} n.  
$ 
Also, with probability at least $1 - c$ for a sufficiently small absolute constant $c$,
$ 
    \left|X_q - \sum_{i\in S_q} Y_i \right| \lesssim U'_q :=  \left(\frac{p/\sigma_p}{q/\sigma_q}\right)^{2t} \left(\frac{\sigma_p}{\sigma_q} n +  \sqrt{n} \right).  
$ 
\end{lemma}

\begin{lemma}\label{cl:spt}
Under the same conditions as in Lemma~\ref{cl:sqt}, for $i \in S_p$,
\begin{align*}
    \mathbb{E}[Y_i] & \lesssim \frac{p/\sigma_p}{q/\sigma_q} \min\left\{1, \frac{L^2}{\sigma_p^2} \right\},
    \quad
    \mathbb{E}[Y_i^2]  \simeq \frac{L^2 \sigma_p^2}{\sigma_q^4} +  \frac{L^4}{\sigma_q^4} +  \left(\frac{p/\sigma_p}{q/\sigma_q} \right)^2 \min\left\{1, \frac{L^2}{\sigma_p^2}\right\} +  \frac{p/\sigma_p}{q/\sigma_q}  \frac{L^2}{\sigma_q^2}. 
\end{align*}
And with probability at least $1 - c$ for a sufficiently small absolute constant $c$,
$ 
    \left|X_p - \sum_{i\in S_p} Y_i \right| \lesssim U_p :=  \left(\frac{p/\sigma_p}{q/\sigma_q}\right)^{2t} pn.  
$ 
\end{lemma}

Now define
$$
Z_i = Y_i - \mathbb{E}[Y_i], \quad
Z = \frac{1}{\sqrt{M_2 |S_q|}} \sum_{i \in S_q} Z_i.
$$
To apply the Berry-Essen Theorem, we bound the first three moments of $Z_i$. Clearly, $\mathbb{E}[Z_i]=0$. 
\begin{lemma} \label{cl:z}
Under the same conditions as in Lemma~\ref{cl:sqt}, for $i \in S_q$,
\begin{align*}
    M_2 & := \mathbb{E}[Z_i^2] 
    \simeq   \frac{p^2/\sigma_p}{q^2/\sigma_q} \min\left\{1, \frac{L^2}{\sigma_p^2}\right\} + \frac{L^2}{\sigma_q^2},
    \\
    M_3 & := \mathbb{E}[|Z_i |^3] 
    \lesssim 
    \frac{p^3/\sigma_p^2}{q^3 /\sigma_q^2} \min\left\{1, \frac{L^2}{\sigma_p^2}\right\} 
    + \frac{p^2/ \sigma_p^2}{q^2 \sigma_q} L^2
    \\
    & \quad + \frac{p/\sigma_p}{q \sigma_q^4} L^2 (\sigma_p^3 + \sigma_p^2 L + \sigma_p L^2) + \frac{L^3}{\sigma_q^3}.
\end{align*}
\end{lemma}
By the Berry-Essen Theorem, conditioned on $S_q$, the CDF $F(t)$ of $Z$
satisfies 
$ 
    |F(t) - \Phi(t)| \lesssim \frac{M_3}{ \sqrt{M_2^{3}|S_q|} }
$ 
where $\Phi(t)$ is the CDF of a standard normal distribution.
By the Chernoff's bound, with probability $1- n^{\Theta(1)}$, $|S_p| \simeq pn, |S_q| \simeq qn \simeq n$. Assume this is true in the rest of the proof.

Now we consider different cases for $p$ and set $\sigma_p, \sigma_q$ and $L$ accordingly.

\textbf{Case 1.} 
Suppose $p\ge \Omega(\ln n/n)$ and $p \le \frac{c_p}{n^{3/4}}$ for some sufficiently small constant $c_p>0$. Then set $\sigma_p =1$, $\sigma_q = C_\sigma/(p^2 n)$ and 
$L= c_L/(p^2 n^{3/2})\simeq \sigma_q/\sqrt{n}$
for some sufficiently large constant $C_\sigma>0$ and some sufficiently small constant $c_L>0$. Set $t=1$. Then
\begin{align*}
    &M_2 \simeq p^2 \sigma_q + \frac{L^2 }{\sigma_q^2} \simeq \frac{1}{n}.
    \\
    &M_3 \lesssim p^3 \sigma_q^2 + p^2 \frac{L^2}{\sigma_q} +  \frac{p}{\sigma_q^4} L^2 (1 + L + L^2) + \frac{L^3}{\sigma_q^3} \simeq \frac{1}{pn^2} + \frac{1}{n^{3/2}}.
\end{align*}
Then conditioned on $|S_q| > qn/2 > n/4$, we have $\frac{M_3}{ \sqrt{M_2^3|S_q|}}  \lesssim \frac{1}{pn} = o(1)$. 
Then we have for constants $C_Z >0$ and $c_z>0$, 
$ 
    \Pr[Z \le - C_Z] = \Pr\left[\sum_{i \in S_q} Z_i \le - C_Z \sqrt{M_2 |S_q|} \right] \ge c_z.
$ 
So with a constant probability, $- \sum_{i \in S_q} Z_i \ge  C_Z\sqrt{M_2n} \simeq C_Z$.
We also have with probability $1-c$ for a sufficiently small absolute constant $c$,
\begin{align*}
    \sum_{i \in S_q} \mathbb{E}[Y_i] & \lesssim \frac{L^2}{\sigma_q^2} qn  \simeq 1, 
    \\
    U_q & = \left(\frac{p/\sigma_p}{q/\sigma_q}\right)^{2} \frac{\sigma_p}{\sigma_q} n  \simeq \frac{p^2/\sigma_p}{q^2/\sigma_q}n 
    \simeq 1,
    \\
    U_p & =  \left(\frac{p/\sigma_p}{q/\sigma_q}\right)^{2} pn \simeq \frac{1}{pn} = o(1),
    \\
    \sum_{i \in S_p} \mathbb{E}[Y_i] & \lesssim pn \frac{p/\sigma_p}{q/\sigma_q} \simeq p^2 \sigma_q n \simeq 1, 
    \\
    \sum_{i \in S_p} \mathbb{E}[Y_i^2] & \lesssim pn \left(\frac{L^2 \sigma_p^2}{\sigma_q^4} + \frac{L^4}{\sigma_q^4} + \left(\frac{p/\sigma_p}{q/\sigma_q} \right)^2 +  \frac{p/\sigma_p}{q/\sigma_q}  \frac{L^2}{\sigma_q^2} \right)
    \\
    & \lesssim pn \left( \frac{1}{\sigma_q^2 n } + \frac{1}{n^2} + p^2\sigma_q^2 + \frac{p\sigma_q }{n}\right) \lesssim \frac{1}{pn} = o(1). 
\end{align*}

\yingyu{need to show the terms are sufficiently small by tuning the constants $c_p, c_L, C_\sigma$ etc.}
Therefore, with a constant probability, 
$\ln \frac{\mathcal{L}_+}{\mathcal{L}_-} $ is negative. 
The expected error $\mathbb{E}|\hat{\mu} - \mu^*|$ of any estimator $\hat{\mu}$ is $\Omega(L) = \Omega(1/(p^2 n^{3/2})) = \Omega(\sqrt{n}/m^2)$.

\textbf{Case 2.} 
Suppose $p \ge \frac{C_p}{n^{3/4}}$ and $p < \frac{c_p}{n^{2/t}}$ for some sufficiently large absolute constant $C_p$ and sufficiently small absolute constant $c_p$.
Then set $\sigma_p =1$, $\sigma_q = C_\sigma /p^{2/3}$ and 
$L= c_L/(p^{2/3} n^{1/2}) \simeq \sigma_q/\sqrt{n}$
for some sufficiently large constant $C_\sigma>0$ and some sufficiently small constant $c_L>0$.
Then
\begin{align*}
    &M_2 \simeq p^2 \sigma_q L^2 + \frac{L^2 }{\sigma_q^2} \simeq \frac{1}{n}.
    \\
    &M_3 \lesssim p^3 \sigma_q^2 L^2 + p^2 \frac{L^2}{\sigma_q} +  \frac{p}{\sigma_q^4} L^2 + \frac{L^3}{\sigma_q^3} \simeq \frac{p^{1/3}}{n} + \frac{1}{n^{3/2}}.
\end{align*}
Then conditioned on $|S_q| > qn/2 > n/4$, we have $\frac{M_3}{ \sqrt{M_2^3|S_q|}}  = o(1)$.
Then we have for constants $C_Z >0$ and $c_z>0$, 
$ 
    \Pr[Z \le - C_Z] = \Pr\left[\sum_{i \in S_q} Z_i \le - C_Z \sqrt{M_2 |S_q|} \right] \ge c_z.
$ 
So with a constant probability, $- \sum_{i \in S_q} Z_i \ge C_Z \sqrt{M_2n} \simeq C_Z$.
We also have with probability $1-c$ for a sufficiently small absolute constant $c$,
\begin{align*}
    \sum_{i \in S_q} \mathbb{E}[Y_i] & \lesssim \frac{L^2}{\sigma_q^2} qn  \simeq 1, 
    \\
    U'_q & = \left(\frac{p/\sigma_p}{q/\sigma_q}\right)^{2t} \left(\frac{\sigma_p}{\sigma_q} n + \sqrt{n} \right)  
    \lesssim 1,
    \\
    U_p & = \left(\frac{p/\sigma_p}{q/\sigma_q}\right)^{2t} pn \lesssim 1,
    \\
    \sum_{i \in S_p} \mathbb{E}[Y_i] & \lesssim pn \frac{p/\sigma_p}{q/\sigma_q}L^2  \simeq 1, 
    \\
    \sum_{i \in S_p} \mathbb{E}[Y_i^2] & \lesssim pn \left(\frac{L^2 \sigma_p^2}{\sigma_q^4} +
    \frac{L^4}{\sigma_q^4} +
    \left(\frac{p/\sigma_p}{q/\sigma_q} \right)^2 L^2 +  \frac{p/\sigma_p}{q/\sigma_q}  \frac{L^2}{\sigma_q^2} \right)
    \\
    & \lesssim pn \left( \frac{1}{\sigma_q^2 n } +
    \frac{1}{n^2} +
    p^2\sigma_q^2 L^2 + \frac{p\sigma_q }{n}\right) \lesssim p^{1/3} = o(1).  
\end{align*}
Therefore, with a constant probability, $\ln \frac{\mathcal{L}_+}{\mathcal{L}_-} $ is negative. 
The expected error $\mathbb{E}|\hat{\mu} - \mu^*|$ of any estimator $\hat{\mu}$ is $\Omega(L) = \Omega(1/(p^{2/3} n^{1/2})) = \Omega(n^{1/6}/m^{2/3})$.

\section{Conclusion} \label{sec:conclusion}

This work considered mean estimation in the setting of entangled single-sampled Gaussians where given one sample from each of $n$ Gaussians with a common mean but different variances, the goal is to learn the mean. It studied the subset-of-signals model where an unknown subset of $m$ variances are bounded, and proved upper and lower bounds, which are summarized in Figure~\ref{fig:bounds}. A natural future direction is to close the gap between the upper bound and the lower bound.

\newpage
\section*{Acknowledgement}
This work was supported in part by FA9550-18-1-0166. The authors would also like to acknowledge the support provided by the University of Wisconsin-Madison Office of the Vice Chancellor for Research and Graduate Education with funding from the Wisconsin Alumni Research Foundation. 

\bibliography{ref}

\begin{thebibliography}{23}
\providecommand{\natexlab}[1]{#1}
\providecommand{\url}[1]{\texttt{#1}}
\expandafter\ifx\csname urlstyle\endcsname\relax
  \providecommand{\doi}[1]{doi: #1}\else
  \providecommand{\doi}{doi: \begingroup \urlstyle{rm}\Url}\fi

\bibitem[Achlioptas and McSherry(2005)]{achlioptas2005spectral}
Dimitris Achlioptas and Frank McSherry.
\newblock On spectral learning of mixtures of distributions.
\newblock In \emph{International Conference on Computational Learning Theory},
  pages 458--469. Springer, 2005.

\bibitem[Belkin and Sinha(2010{\natexlab{a}})]{belkin2010polynomial}
Mikhail Belkin and Kaushik Sinha.
\newblock Polynomial learning of distribution families.
\newblock In \emph{2010 IEEE 51st Annual Symposium on Foundations of Computer
  Science}, pages 103--112. IEEE, 2010{\natexlab{a}}.

\bibitem[Belkin and Sinha(2010{\natexlab{b}})]{belkin2010toward}
Mikhail Belkin and Kaushik Sinha.
\newblock Toward learning gaussian mixtures with arbitrary separation.
\newblock In \emph{COLT}, pages 407--419. Citeseer, 2010{\natexlab{b}}.

\bibitem[Cheng et~al.(2019)Cheng, Diakonikolas, and Ge]{cheng2019high}
Yu~Cheng, Ilias Diakonikolas, and Rong Ge.
\newblock High-dimensional robust mean estimation in nearly-linear time.
\newblock In \emph{Proceedings of the Thirtieth Annual ACM-SIAM Symposium on
  Discrete Algorithms}, pages 2755--2771. SIAM, 2019.

\bibitem[Chierichetti et~al.(2014)Chierichetti, Dasgupta, Kumar, and
  Lattanzi]{chierichetti2014learning}
Flavio Chierichetti, Anirban Dasgupta, Ravi Kumar, and Silvio Lattanzi.
\newblock Learning entangled single-sample gaussians.
\newblock In \emph{Proceedings of the twenty-fifth annual ACM-SIAM symposium on
  Discrete algorithms}, pages 511--522. Society for Industrial and Applied
  Mathematics, 2014.

\bibitem[Dasgupta(1999)]{dasgupta1999learning}
Sanjoy Dasgupta.
\newblock Learning mixtures of gaussians.
\newblock In \emph{40th Annual Symposium on Foundations of Computer Science
  (Cat. No. 99CB37039)}, pages 634--644. IEEE, 1999.

\bibitem[Diakonikolas et~al.(2017)Diakonikolas, Kamath, Kane, Li, Moitra, and
  Stewart]{diakonikolas2017being}
Ilias Diakonikolas, Gautam Kamath, Daniel~M Kane, Jerry Li, Ankur Moitra, and
  Alistair Stewart.
\newblock Being robust (in high dimensions) can be practical.
\newblock In \emph{Proceedings of the 34th International Conference on Machine
  Learning-Volume 70}, pages 999--1008. JMLR. org, 2017.

\bibitem[Diakonikolas et~al.(2018)Diakonikolas, Kamath, Kane, Li, Moitra, and
  Stewart]{diakonikolas2018robustly}
Ilias Diakonikolas, Gautam Kamath, Daniel~M Kane, Jerry Li, Ankur Moitra, and
  Alistair Stewart.
\newblock Robustly learning a gaussian: Getting optimal error, efficiently.
\newblock In \emph{Proceedings of the Twenty-Ninth Annual ACM-SIAM Symposium on
  Discrete Algorithms}, pages 2683--2702. Society for Industrial and Applied
  Mathematics, 2018.

\bibitem[Diakonikolas et~al.(2019)Diakonikolas, Kamath, Kane, Li, Moitra, and
  Stewart]{diakonikolas2019robust}
Ilias Diakonikolas, Gautam Kamath, Daniel Kane, Jerry Li, Ankur Moitra, and
  Alistair Stewart.
\newblock Robust estimators in high-dimensions without the computational
  intractability.
\newblock \emph{SIAM Journal on Computing}, 48\penalty0 (2):\penalty0 742--864,
  2019.

\bibitem[Hong et~al.(2018{\natexlab{a}})Hong, Balzano, and
  Fessler]{hong2018asymptotic}
David Hong, Laura Balzano, and Jeffrey~A Fessler.
\newblock Asymptotic performance of pca for high-dimensional heteroscedastic
  data.
\newblock \emph{Journal of multivariate analysis}, 167:\penalty0 435--452,
  2018{\natexlab{a}}.

\bibitem[Hong et~al.(2018{\natexlab{b}})Hong, Fessler, and
  Balzano]{hong2018optimally}
David Hong, Jeffrey~A Fessler, and Laura Balzano.
\newblock Optimally weighted pca for high-dimensional heteroscedastic data.
\newblock \emph{arXiv preprint arXiv:1810.12862}, 2018{\natexlab{b}}.

\bibitem[Huber(2011)]{huber2011robust}
Peter~J Huber.
\newblock \emph{Robust statistics}.
\newblock Springer, 2011.

\bibitem[Kalai et~al.(2010)Kalai, Moitra, and Valiant]{kalai2010efficiently}
Adam~Tauman Kalai, Ankur Moitra, and Gregory Valiant.
\newblock Efficiently learning mixtures of two gaussians.
\newblock In \emph{Proceedings of the forty-second ACM symposium on Theory of
  computing}, pages 553--562. ACM, 2010.

\bibitem[Kannan et~al.(2005)Kannan, Salmasian, and Vempala]{kannan2005spectral}
Ravindran Kannan, Hadi Salmasian, and Santosh Vempala.
\newblock The spectral method for general mixture models.
\newblock In \emph{International Conference on Computational Learning Theory},
  pages 444--457. Springer, 2005.

\bibitem[Moitra and Valiant(2010)]{moitra2010settling}
Ankur Moitra and Gregory Valiant.
\newblock Settling the polynomial learnability of mixtures of gaussians.
\newblock In \emph{2010 IEEE 51st Annual Symposium on Foundations of Computer
  Science}, pages 93--102. IEEE, 2010.

\bibitem[Munoz et~al.(1986)Munoz, Rosner, and Carey]{munoz1986regression}
Alvaro Munoz, Bernard Rosner, and Vincent Carey.
\newblock Regression analysis in the presence of heterogeneous intraclass
  correlations.
\newblock \emph{Biometrics}, pages 653--658, 1986.

\bibitem[Pensia et~al.(2019)Pensia, Jog, and Loh]{pensia2019estimating}
Ankit Pensia, Varun Jog, and Po-Ling Loh.
\newblock Estimating location parameters in entangled single-sample
  distributions.
\newblock \emph{arXiv preprint arXiv:1907.03087}, 2019.

\bibitem[Sanjeev and Kannan(2001)]{sanjeev2001learning}
Arora Sanjeev and Ravi Kannan.
\newblock Learning mixtures of arbitrary gaussians.
\newblock In \emph{Proceedings of the thirty-third annual ACM symposium on
  Theory of computing}, pages 247--257. ACM, 2001.

\bibitem[Valiant(1985)]{valiant1985learning}
Leslie~G Valiant.
\newblock Learning disjunction of conjunctions.
\newblock In \emph{IJCAI}, pages 560--566. Citeseer, 1985.

\bibitem[Vershynin(2018)]{vershynin2018high}
Roman Vershynin.
\newblock \emph{High-dimensional probability: An introduction with applications
  in data science}.
\newblock Cambridge University Press, 2018.

\bibitem[Vicari and Vichi(2013)]{vicari2013multivariate}
Donatella Vicari and Maurizio Vichi.
\newblock Multivariate linear regression for heterogeneous data.
\newblock \emph{Journal of Applied Statistics}, 40\penalty0 (6):\penalty0
  1209--1230, 2013.

\bibitem[Yuan and Liang(2020)]{yuan2020learning}
Hui Yuan and Yingyu Liang.
\newblock Learning entangled single-sample distributions via iterative
  trimming.
\newblock In \emph{The 23rd International Conference on Artificial Intelligence
  and Statistics}, 2020.

\bibitem[Zhang et~al.(2018)Zhang, Cai, and Wu]{zhang2018heteroskedastic}
Anru Zhang, T~Tony Cai, and Yihong Wu.
\newblock Heteroskedastic pca: Algorithm, optimality, and applications.
\newblock \emph{arXiv preprint arXiv:1810.08316}, 2018.

\end{thebibliography}

\clearpage
\appendix

\section{Proofs for Upper Bound} \label{app:upperbound}

\subsection{Proof of Lemma~\ref{lem:uniform}} \label{app:lem:uniform}

Note that $z_i(\mu)$ is 1-Lipschitz w.r.t.\ $\mu$. So a standard $\epsilon$-net argument over the interval $[\mu^\star - \delta_e, \mu^\star + \delta_e]$ gives the bound. 

More precisely, let $\mathcal{X}$ be an $\epsilon$-net over $[\mu^\star - \delta_e, \mu^\star + \delta_e]$, with $\epsilon = \delta_e/n^2$. A standard construction gives $|\mathcal{X}| < 2n^2$. For a fixed $\mu \in \mathcal{X}$, we have
\begin{gather}
    \mu - \delta - \mu^\star \le \mathbb{E} z_i(\mu) \le \mu + \delta  - \mu^\star, 
    \\
    -2\delta \le \bar{z}_i(\mu) \le 2\delta.
\end{gather}
Since $\bar{z}_i(\mu)$ is bounded, we have by sub-Gaussian properties (see, e.g., Section 2.5 and 2.6 of~\cite{vershynin2018high}),
\begin{align}
    \|\bar{z}_i(\mu)\|_{\psi_2}  \lesssim \delta,
\end{align}
and we have with probability at least $1 - 1/n^6$,  for the fixed $\mu$,
\begin{align}
    \left|\sum_{i=1}^n \bar{z}_i(\mu) \right| \lesssim \delta \sqrt{n \ln n}.
\end{align} 
Taking a union bound over $\mathcal{X}$, we have with probability at least $1 - 1/n^3$,  for all $\mu \in \mathcal{X}$,
\begin{align}
    \left|\sum_{i=1}^n \bar{z}_i(\mu) \right| \lesssim \delta\sqrt{n \ln n}. 
\end{align} 
For any $\mu' \not\in \mathcal{X}$, there is $\mu \in \mathcal{X}$ satisfying $|\mu' - \mu| \le \epsilon$. Therefore,
\begin{align}
    \left|\sum_{i=1}^n \bar{z}_i(\mu') \right|
    & \le \left|\sum_{i=1}^n \bar{z}_i(\mu) \right|
    + \left|\sum_{i=1}^n \bar{z}_i(\mu') - \bar{z}_i(\mu) \right|
    \\
    & \le \left|\sum_{i=1}^n \bar{z}_i(\mu) \right|
    + \left|\sum_{i=1}^n z_i(\mu') - z_i(\mu) \right|
    + \left|\sum_{i=1}^n \mathbb{E}[z_i(\mu') - z_i(\mu)] \right|
    \\
    & \le \left|\sum_{i=1}^n \bar{z}_i(\mu) \right|
    + \epsilon n + \epsilon n
    \\
    & \lesssim \delta\sqrt{n \ln n} + \delta_e/n.
\end{align}
This completes the proof.

\subsection{Proof of Lemma~\ref{lem:bias}} \label{app:lem:bias}
W.L.O.G., suppose $\mu \ge \mu^\star$, and let $\delta_e = |\mu - \mu^\star|$. Let $z_i$ be a shorthand for $z_i(\mu)$. Let $g(x) = \frac{1}{\sqrt{2\pi}} \exp\left( - \frac{x^2}{2}\right), a_i = \frac{\delta_e - \delta}{\sigma_i}, b_i = \frac{ \delta_e + \delta}{\sigma_i}$.
Then 
\begin{align} 
  |\mathbb{E} z_i|
  & = \mathbb{E} z_i
  \\
  & =  (\delta_e - \delta ) \int_{-\infty}^{\delta_e - \delta } \frac{1}{\sqrt{2\pi} \sigma_i} \exp\left( - \frac{x^2}{2\sigma_i^2}\right)  dx 
  \\
  & \quad + \int_{\delta_e - \delta }^{\delta_e + \delta } x \cdot \frac{1}{\sqrt{2\pi} \sigma_i} \exp\left( - \frac{x^2}{2\sigma_i^2}\right) dx 
  \\
  & \quad + (\delta_e + \delta ) \int_{\delta_e + \delta }^{+\infty}  \frac{1}{\sqrt{2\pi} \sigma_i} \exp\left( - \frac{x^2}{2\sigma_i^2}\right) dx 
  \\
  & = (\delta_e - \delta ) \int_{-\infty}^{a_i } g(x) dx 
  \\
  & \quad +  \int_{a_i }^{b_i} \sigma_i x g(x) dx 
  \\
  & \quad + (\delta_e + \delta ) \int_{b_i}^{+\infty} g(x) dx
  \\
  & = (\delta_e - \delta ) \int_{-b_i}^{a_i} g(x) dx + (\delta_e - \delta ) \int_{-\infty}^{-b_i} g(x) dx 
  \\
  & \quad + \int_{-a_i }^{b_i} \sigma_i x g(x) dx 
  \\
  & \quad + (\delta_e + \delta ) \int_{b_i}^{+\infty} g(x) dx
  \\
  & = (\delta_e - \delta ) \int_{-a_i}^{b_i} g(x) dx + (\delta_e - \delta ) \int_{b_i}^{+\infty} g(x) dx 
  \\
  & \quad + \int_{-a_i }^{b_i} \sigma_i x g(x) dx 
  \\
  & \quad + (\delta_e + \delta ) \int_{b_i}^{+\infty} g(x) dx
  \\
  & =  \int_{-a_i }^{b_i} [\sigma_i x + (\delta_e - \delta ) ]  g(x) dx 
  + 2\delta_e \int_{b_i}^{+\infty} g(x) dx.
\end{align} 
We consider two cases.

\textbf{Case 1:} $\delta \ge \delta_e$. 
Then $-a_i \ge 0$. 
\begin{align}
& \int_{-a_i }^{b_i} [\sigma_i x + (\delta_e - \delta ) ]  g(x) dx 
\\
= &  \int_{-a_i }^{b_i}\delta_e g(x) dx  + \int_{-a_i }^{\frac{b_i-a_i}{2}} [\sigma_i x - \delta  ]  g(x) dx + \int_{\frac{b_i-a_i}{2}}^{b_i} [\sigma_i x - \delta ]  g(x) dx 
\\
= &  \int_{-a_i }^{b_i}\delta_e g(x) dx  - \int_{\frac{b_i-a_i}{2}}^{b_i} [\sigma_i y - \delta  ]  g\left(\frac{2\delta}{\sigma_i} - y \right) dy + \int_{\frac{b_i-a_i}{2}}^{b_i} [\sigma_i x - \delta ]  g(x) dx 
\\
\le &  \int_{-a_i }^{b_i}\delta_e g(x) dx.
\end{align}
Therefore,
\begin{align} 
  |\mathbb{E} z_i| & \le  \delta_e \int_{-a_i }^{b_i}g(x) dx + 2\delta_e \int_{b_i}^{+\infty} g(x) dx
  \\
  & =  \delta_e \left(1 - \int_{-b_i }^{-a_i}g(x) dx \right)
  \\
  & = \delta_e \left(1 - \int_{0}^{b_i}g(x) dx \right).
\end{align}

\textbf{Case 2:} $\delta < \delta_e$. Then $a_i > 0$.
\begin{align}
& \int_{-a_i }^{b_i} [\sigma_i x + (\delta_e - \delta ) ]  g(x) dx 
\\
= & \int_{-a_i }^{a_i} [\sigma_i x + (\delta_e - \delta ) ]  g(x) dx  + \int_{a_i }^{b_i} [\sigma_i x + (\delta_e - \delta ) ]  g(x) dx
\\
= & \int_{-a_i }^{a_i} (\delta_e - \delta ) g(x) dx  + \int_{a_i }^{b_i} [\sigma_i x + (\delta_e - \delta ) ]  g(x) dx.
\end{align}
Then second term can be bounded as in Case 1. 
\begin{align}
\int_{a_i }^{b_i} [\sigma_i x + (\delta_e - \delta ) ]  g(x) dx
\le &  \int_{a_i }^{b_i}\delta_e g(x) dx.
\end{align}
Therefore,
\begin{align} 
  |\mathbb{E} z_i| & \le  \int_{-a_i }^{a_i} (\delta_e - \delta ) g(x) dx + \delta_e \int_{a_i }^{b_i}g(x) dx + 2\delta_e \int_{b_i}^{+\infty} g(x) dx
  \\
  & =  \int_{-a_i }^{a_i} (\delta_e - \delta ) g(x) dx + \delta_e - \delta_e \int_{-b_i }^{a_i}g(x) dx 
  \\
  & = - \delta\int_{-a_i }^{a_i}  g(x) dx + \delta_e - \delta_e \int_{-b_i }^{-a_i}g(x) dx 
  \\
  & \le \delta_e - \delta \int_{-b_i }^{a_i}g(x) dx 
  \\
  & \le \delta_e \left(1 - \frac{\delta}{\delta_e}\int_{0 }^{b_i}g(x) dx \right).
\end{align}
In summary, for both cases, we have
\begin{align} 
  |\mathbb{E} z_i| 
  & \le \delta_e \left(1 - \frac{\delta}{\max\{\delta_e, \delta\}}\int_{0 }^{b_i}g(x) dx \right).
\end{align}
 
To simplify the bound, we consider two cases.
If $\sigma_i \le \delta + \delta_e$, then $b_i \ge 1$, and
\begin{align} 
   \int_0^{b_i}g(x) dx  \ge \int_0^{1}g(x) dx \ge 1/2. 
\end{align}
If $\sigma_i > \delta + \delta_e$, then
\begin{align} 
  \int_0^{b_i}g(x) dx  
  & \ge  b_i g(b_i)
  \\
  & = \frac{1}{\sqrt{2\pi}} \frac{\delta  + \delta_e}{\sigma_i} \exp\left( -\frac{(\delta  + \delta_e)^2}{2\sigma_i^2}\right)
  \\
  & \ge \frac{\delta  + \delta_e}{5\sigma_i} \ge \frac{\delta }{5\sigma_i}.
\end{align}
Then for both cases, we have
\begin{align} 
  |\mathbb{E} z_i| 
  & \le \delta_e \left(1 - \frac{\delta}{\max\{\delta_e, \delta\}} \frac{\delta }{5\max\{\sigma_i,\delta \}} \right).
\end{align}

\section{Proofs for Lower Bound} \label{app:lowerbound}

For convenience, define 
\begin{align}
    \alpha & = \frac{p/\sigma_p}{q/\sigma_q},
    \\
    \beta & = \frac{2L}{\sigma_q^2}, 
    \\
    \gamma & = \frac{\sigma_p}{\sigma_q},
    \\
    \frac{1}{\sigma^2_{pq}} & = \frac{1}{\sigma^2_p} - \frac{1}{\sigma^2_q}.
\end{align}
Then we have
\begin{align}
    N_i & = \alpha \exp\left(-\frac{(x_i-L)^2}{2\sigma^2_{pq}} \right)
    \\
    D_i & = \alpha \exp\left(-\frac{(x_i+L)^2}{2\sigma^2_{pq}}  \right)
\end{align}
and 
\begin{align}
    \ln \frac{\mathcal{L}_+}{\mathcal{L}_-} 
    = & \sum_{i=1}^{n} \ln \frac{1 + N_i}{1 + D_i} + \beta x_i 
    \\
    = & \underbrace{
    \sum_{i\in S_p} \left( \ln \frac{1 + N_i}{1 + D_i} 
    + \beta x_i \right) 
    }_{X_p} + 
    \underbrace{
    \sum_{i\in S_q} \left( \ln \frac{1 + N_i}{1 + D_i} + \beta x_i \right)
    }_{X_q}.
\end{align}

\subsection{Proof of Lemma~\ref{cl:sqt}}

\begin{lemma} \label{lem:nd}
Suppose the mean is $L$. For a positive integer $j$, $N_i^j$ and $D_i^j$ are sub-Gaussian with norms
\begin{align}
    \|N_i^j \|_{\psi_2}  \lesssim \left(\frac{p/\sigma_p}{q/\sigma_q}\right)^j, \quad
    \|D_i^j \|_{\psi_2}  \lesssim \left(\frac{p/\sigma_p}{q/\sigma_q}\right)^j.
\end{align}
\end{lemma}
\begin{proof}
Recall that if the moments of a random variable $X$ satisfy 
$\| X \|_{L_p} = (\mathbb{E}|X|^p)^{1/p} \le K \sqrt{p}$ for all $p \ge 1$, then $\|X \|_{\psi_2}  \lesssim K$. 
The lemma then follows from Lemma~\ref{lem:int_e}.
\end{proof}

Since for any $x > 0$,
\begin{align}
    \sum_{j=1}^{2t} \frac{(-1)^{j+1} x^j}{j} \le \ln(1+x) \le \sum_{j=1}^{2t-1} \frac{(-1)^{j+1} x^j}{j},
\end{align}
we have 
\begin{align}
    V_i - \frac{N_i^{2t}}{2t}  \le \ln \frac{1+N_i}{1+D_i} \le V_i + \frac{D_i^{2t}}{2t},
\end{align}
and thus
\begin{align}
    \left|\sum_{i\in S_q} \ln\frac{1 + N_i}{1 + D_i} + \beta x_i - \sum_{i\in S_q} Y_i \right| \le \sum_{i \in S_q} \max\left\{\frac{N_i^{2t}}{2t}, \frac{D_i^{2t}}{2t}\right\}.
\end{align}
By Lemma~\ref{lem:int_e} and \ref{lem:int_ee}, for $i \in S_q$,
\begin{align}
    \mathbb{E}[N^j_i] & = \alpha^j \frac{\sigma_{pq}}{\sqrt{\sigma_{pq}^2 + j\sigma_q^2}},
    \\
    \mathbb{E}[D^j_i] & = \alpha^j \frac{\sigma_{pq}}{\sqrt{\sigma_{pq}^2 + j\sigma_q^2}} \exp\left( - \frac{2j L^2}{\sigma_{pq}^2 + j\sigma_q^2}\right) \le \mathbb{E}[N^j_i],
    \\
    \mathbb{E}[N_i^j D_i^j] & = \alpha^{2j} \frac{\sigma_{pq}}{\sqrt{\sigma_{pq}^2 + 2 j\sigma_q^2}} \exp\left(
    -2jL^2 \frac{\sigma_{pq}^2 + j \sigma_q^2}{\sigma_{pq}^4 + 2 j \sigma_{pq}^2 \sigma_q^2}\right).
\end{align}
By the Chernoff's bound, with probability $1- n^{\Theta(1)}$, $|S_q| \simeq qn \simeq n$. 
Conditioned on $S_q$, we have with probability at least $1-e^{\Theta(\alpha^{2t} \gamma n)}$, 
\begin{align}
      \max\left\{\sum_{i \in S_q}\frac{N_i^{2t}}{2t}, \sum_{i \in S_q} \frac{D_i^{2t}}{2t}\right\} \le 2 \sum_{i \in S_q} \mathbb{E}  N_i^{2t} \simeq \alpha^{2t} \gamma n.
\end{align}

So with probability $1- n^{\Theta(1)} - e^{\Theta(\alpha^{2t} \gamma n)}$, 
\begin{align}
     \left|\sum_{i\in S_q} \ln\frac{1 + N_i}{1 + D_i} + \beta x_i - \sum_{i\in S_q} Y_i \right| \lesssim  \alpha^{2t} \gamma n.
\end{align}



Now consider $Y_i$. 
Since $p$ is sufficiently small compared to $q$ and $L$ is sufficiently small compared to $\sigma_q$, and $\alpha < c_\alpha$ for some sufficiently small absolute constant $c_\alpha < 1$, we have 
\begin{align}
    \mathbb{E}[Y_i] \lesssim \beta  L + \sum_{j=1}^{2t-1} \alpha^j \gamma \frac{L^2}{\sigma_q^2}  & \lesssim \frac{L^2}{\sigma_q^2} \left( 1 + \sum_{j=1}^{2t-1} \alpha^j \gamma \right)
    \\
    & \lesssim 
    \frac{L^2}{\sigma_q^2}.
\end{align}
Let $V_{ij} = (-1)^{j+1}(N_i^j - D_i^j)/j$, then $Y_i = \beta x_i +  \sum_{j=1}^{2t-1} V_{ij}$. By Lemma~\ref{lem:int_e},  \ref{lem:int_ee}, \ref{lem:int_x_e},  and that $\alpha < c_\alpha$ for some sufficiently small absolute constant $c_\alpha < 1$,
\begin{align}
    \mathbb{E}[Y^2_i] & = \mathbb{E}\left[ \beta^2 x_i^2 + \sum_{j=1}^{2t-1} V_{ij}^2 + 2 \sum_{j=1}^{2t-1} \beta x_i V_{ij} + 2 \sum_{j<k; j,k=1}^{2t-1} V_{ij}V_{ik} \right]
    \\
    & \simeq
    \beta^2 \sigma_q^2 + \sum_{j=1}^{2t-1} \alpha^{2j} \gamma \min\left\{1, \frac{L^2}{\sigma_p^2}\right\}
    \\
    & \quad + \sum_{j=1}^{2t-1} (-1)^{j+1}\alpha^j \beta \gamma L + \sum_{j<k; j,k=1}^{2t-1} (-1)^{j+k} \alpha^{j+k} \gamma \min\left\{1, \frac{L^2}{\sigma_p^2}\right\}
    \\
    & \simeq
    \beta^2 \sigma_q^2 + \alpha^{2} \gamma \min\left\{1, \frac{L^2}{\sigma_p^2}\right\} + \alpha \beta \gamma L
    \\
    & \simeq \frac{L^2}{\sigma_q^2} + \frac{p^2/\sigma_p}{q^2/\sigma_q} \min\left\{1, \frac{L^2}{\sigma_p^2}\right\} +  \frac{p}{q} \frac{L^2}{\sigma_q^2}
    \\
    & \simeq \frac{p^2/\sigma_p}{q^2/\sigma_q} \min\left\{1, \frac{L^2}{\sigma_p^2}\right\} + \frac{L^2}{\sigma_q^2}
\end{align}
where the last line follows from $p<q$.

\subsection{Proof of Lemma~\ref{cl:spt}}
The proof is similar to that of Lemma~\ref{cl:sqt}.

Again, we have
\begin{align}
    \left|\sum_{i\in S_p} \ln\frac{1 + N_i}{1 + D_i} + \beta x_i - \sum_{i\in S_p} Y_i \right| \le \sum_{i \in S_p} \max\left\{\frac{N_i^{2t}}{2t}, \frac{D_i^{2t}}{2t}\right\}.
\end{align}
By Lemma~\ref{lem:int_e} and \ref{lem:int_ee}, for $i \in S_p$,
\begin{align}
    \mathbb{E}[N^j_i] & = \alpha^j \frac{\sigma_{pq}}{\sqrt{\sigma_{pq}^2 + j\sigma_p^2}},
    \\
    \mathbb{E}[D^j_i] & = \alpha^j \frac{\sigma_{pq}}{\sqrt{\sigma_{pq}^2 + j\sigma_p^2}} \exp\left( - \frac{2j L^2}{\sigma_{pq}^2 + j\sigma_p^2}\right) \le \mathbb{E}[N^j_i],
    \\
    \mathbb{E}[N_i^j D_i^j] & = \alpha^{2j} \frac{\sigma_{pq}}{\sqrt{\sigma_{pq}^2 + 2 j\sigma_p^2}} \exp\left(
    -2jL^2 \frac{\sigma_{pq}^2 + j \sigma_p^2}{\sigma_{pq}^4 + 2 j \sigma_{pq}^2 \sigma_p^2}\right).
\end{align}
Conditioned on $S_p$, by Lemma \ref{lem:nd}, we have with probability at least $1-\delta$, 
\begin{align}
     \sum_{i \in S_p} \max\left\{\frac{N_i^{2t}}{2t}, \frac{D_i^{2t}}{2t}\right\} \lesssim  \frac{\alpha^{2t}}{2t} |S_p|  +  \alpha^{2t} \sqrt{|S_p| \log\frac{1}{\delta}}.
\end{align}
By the Chernoff's bound, with probability $1- n^{\Theta(1)}$, $|S_p| \simeq pn$. So with probability $1- n^{\Theta(1)} - c$ for a sufficiently small absolute constant $c$, 
\begin{align}
     \left|\sum_{i\in S_p} \ln\frac{1 + N_i}{1 + D_i} + \beta x_i - \sum_{i\in S_p} Y_i \right| \lesssim  \alpha^{2t} p n  +  \alpha^{2t} \sqrt{pn}.
\end{align}

Now consider $Y_i$. 
Since $p$ is sufficiently small compared to $q$ and $L$ is sufficiently small compared to $\sigma_q$, and $\alpha < c_\alpha$ for some sufficiently small absolute constant $c_\alpha < 1$, we have 
\begin{align}
    \mathbb{E}[Y_i] \lesssim \beta  L + \sum_{j=1}^{2t-1} \alpha^j  \min\left\{1, \frac{L^2}{\sigma_p^2} \right\}  & \lesssim \frac{L^2}{\sigma_q^2} + \alpha \min\left\{1, \frac{L^2}{\sigma_p^2} \right\}.
\end{align}
Let $V_{ij} = (-1)^{j+1}(N_i^j - D_i^j)/j$, then $Y_i = \beta x_i +  \sum_{j=1}^{2t-1} V_{ij}$. By Lemma~\ref{lem:int_e}, \ref{lem:int_ee}, \ref{lem:int_x_e},  and that $\alpha < c_\alpha$ for some sufficiently small absolute constant $c_\alpha < 1$,
\begin{align}
    \mathbb{E}[Y^2_i] & = \mathbb{E}\left[ \beta^2 x_i^2 + \sum_{j=1}^{2t-1} V_{ij}^2 + 2 \sum_{j=1}^{2t-1} \beta x_i V_{ij} + 2 \sum_{j<k; j,k=1}^{2t-1} V_{ij}V_{ik} \right]
    \\
    & \simeq
    \beta^2 (\sigma_p^2 + L^2) + \sum_{j=1}^{2t-1} \alpha^{2j}  \min\left\{1, \frac{L^2}{\sigma_p^2}\right\}
    \\
    & \quad + \sum_{j=1}^{2t-1} (-1)^{j+1}\alpha^j \beta  L + \sum_{j<k; j,k=1}^{2t-1} (-1)^{j+k} \alpha^{j+k}  \min\left\{1, \frac{L^2}{\sigma_p^2}\right\}
    \\
    & \simeq
    \beta^2 (\sigma_p^2 + L^2) + \alpha^{2}  \min\left\{1, \frac{L^2}{\sigma_p^2}\right\} + \alpha \beta  L
    \\
    & \simeq \frac{L^2 \sigma_p^2}{\sigma_q^4} + \frac{L^4}{\sigma_q^4} + \left(\frac{p/\sigma_p}{q/\sigma_q} \right)^2 \min\left\{1, \frac{L^2}{\sigma_p^2}\right\} +  \frac{p/\sigma_p}{q/\sigma_q}  \frac{L^2}{\sigma_q^2}.
\end{align}
where the last line follows from $p<q$.

\subsection{Proof of Lemma~\ref{cl:z}}

The second moment is
\begin{align}
    M_2 := \mathbb{E}[Z_i^2] & = \mathbb{E}[Y_i^2] - \mathbb{E}^2[Y_i] 
    \\
    & \simeq \frac{p^2/\sigma_p}{q^2/\sigma_q} \min\left\{1, \frac{L^2}{\sigma_p^2}\right\} + \frac{L^2}{\sigma_q^2} - \left(\frac{L^2}{\sigma_q^2}\right)^2
    \\
    & \simeq \frac{p^2/\sigma_p}{q^2/\sigma_q} \min\left\{1, \frac{L^2}{\sigma_p^2}\right\} + \frac{L^2}{\sigma_q^2}.
\end{align}
where the last line follows from $L$ is sufficiently small compared to $\sigma_q$. 

To compute the third moment, let $R_i = \beta x_i - \beta \mathbb{E}[x_i] = \beta (x_i - L)$. Then
\begin{align}
    \mathbb{E}[|Z_i |^3] & = \mathbb{E}[|Y_i - 
    \mathbb{E}[Y_i] |^3] 
    \\
    & = \mathbb{E}[|V_i - \mathbb{E}[V_i] + R_i |^3]  
    \\
    & \le \mathbb{E}[|V_i - \mathbb{E}[V_i]|^3] +
    \mathbb{E}[ |R_i |^3]  
    \\
    & \quad + 3 \mathbb{E}[ |V_i - \mathbb{E}[V_i]|^2 |R_i |]  + 3 \mathbb{E}[ |V_i - \mathbb{E}[V_i]| |R_i |^2].  
\end{align}
The terms can be bounded respectively.
\begin{align}
    \mathbb{E}[|V_i - \mathbb{E}[V_i]|^3] 
    & \le \mathbb{E}[|V_i - \mathbb{E}[V_i]|^2] \max_{x_i} |V_i - \mathbb{E}[V_i]|
    \\
    & \lesssim \mathbb{E}[|V_i - \mathbb{E}[V_i]|^2] \max_{j, x_i} \left| (N^j_i - D^j_i) - \mathbb{E}(N^j_i - D^j_i) \right|
    \\
    & \lesssim \mathbb{E}[|V_i - \mathbb{E}[V_i]|^2] \max_{x_i}\{N_i, D_i\}
    \\
    & \lesssim \mathbb{E}[V_i^2] \max_{x_i}\{N_i, D_i\}
    \\
    & \lesssim \alpha^3 \gamma \min\left\{1, \frac{L^2}{\sigma_p^2}\right\}.
\end{align}

By Lemma~\ref{lem:int_xabs_2} and Lemma~\ref{lem:moment},
\begin{align}
    \mathbb{E}[ |V_i - \mathbb{E}[V_i]|^2 |R_i |]
    & \lesssim 
    \mathbb{E}[ V_i^2|R_i |] +  \mathbb{E}^2[V_i]\mathbb{E}[|R_i |]
    \\
    & \lesssim 
     \sum_{j=1}^{2t-1} \mathbb{E}\left[ (N_i^j - D_i^j)^2 |R_i|\right] +  
     \left (\sum_{j=1}^{2t-1} \mathbb{E}\left[N_i^j - D_i^j\right] \right)^2 \mathbb{E}[|R_i|]
    \\
    & \lesssim 
     \sum_{j=1}^{2t-1} \alpha^{2j} \beta \frac{L}{\sigma_q} +  
     \left (\sum_{j=1}^{2t-1} \alpha^j \gamma \frac{L^2}{\sigma_q^2} \right)^2 \beta \sigma_q
    \\
    & \lesssim 
     \alpha^2 \beta \frac{L}{\sigma_q} +  
     \alpha^2 \gamma^2 \frac{L^5}{\sigma_q^5} 
    \\
    & \lesssim \alpha^2 \beta \frac{L}{\sigma_q}.
\end{align}
For $\mathbb{E}[ |V_i - \mathbb{E}[V_i]| |R_i |^2]$, let $V_{ij} = (-1)^{j+1}(N_i^j - D_i^j)/j$.
\begin{align}
    \mathbb{E}[ |V_i - \mathbb{E}[V_i]| |R_i |^2] & \le \mathbb{E}[ |V_i| |R_i |^2] + |\mathbb{E}[V_i]| \mathbb{E}[ |R_i |^2]
    \\
    & \le \sum_{j=1}^{2t-1}\mathbb{E}[ |V_{ij}| |R_i |^2] + \sum_{j=1}^{2t-1}|\mathbb{E}[V_{ij}]|\mathbb{E}[ |R_i |^2].
\end{align}
For the first part, by Lemma~\ref{lem:int_x2_abs},
\begin{align}
    \mathbb{E}[ |V_{ij}| |R_i |^2] 
    & \lesssim \alpha^j \beta^2 \frac{\sigma_p^3}{\sigma_q} \mathrm{erf}(\Theta(L/\sigma_p)) 
    \\
     & + \alpha^j \beta^2 \frac{\sigma_p (\sigma_p^2 + L^2) }{\sigma_q} \mathrm{erf}(\Theta(L/\sigma_p))\exp\left( -\Theta\left(L^2/\sigma_q^2\right)\right)  
     \\
     & + \alpha^j \beta^2 \frac{\sigma_p^2 L}{\sigma_q}\exp\left( -\frac{L^2}{2\sigma_p^2} \right) 
    \\ 
     & \lesssim \alpha^j \beta^2 \frac{\sigma_p^3 + \sigma_p^2 L + \sigma_p L^2}{\sigma_q}.
\end{align}
For the second part, by Lemma~\ref{lem:moment},
\begin{align}
    \mathbb{E}[V_{ij}]\mathbb{E}[ |R_i |^2]
    \lesssim \alpha^j \gamma \frac{L^2}{\sigma_q^2} \beta^2 \sigma_q^2.
\end{align}
Combining the two parts,
\begin{align}
    \mathbb{E}[ |V_i - \mathbb{E}[V_i]| |R_i |^2] & \le \mathbb{E}[ |V_i| |R_i |^2] + \mathbb{E}[V_i]\mathbb{E}[ |R_i |^2]
    \\
    & \lesssim \sum_{j=1}^{2t-1} \alpha^j \beta^2 \frac{\sigma_p^3 + \sigma_p^2 L + \sigma_p L^2}{\sigma_q} + \sum_{j=1}^{2t-1} \alpha^j \gamma \frac{L^2}{\sigma_q^2} \beta^2 \sigma_q^2
    \\
    & \lesssim \alpha \beta^2 \frac{\sigma_p^3 + \sigma_p^2 L + \sigma_p L^2}{\sigma_q} + \alpha\gamma \frac{L^2}{\sigma_q^2} \beta^2 \sigma_q^2
    \\
    & \lesssim \alpha \beta^2 \frac{\sigma_p^3 + \sigma_p^2 L + \sigma_p L^2}{\sigma_q}.
\end{align}
where the last line follows from $\gamma = \sigma_p/\sigma_q$.
Finally, also by Lemma~\ref{lem:moment},
\begin{align}
    \mathbb{E}[ |R_i |^3]  & \lesssim (\beta \sigma_q)^3.
\end{align}
Combining all terms together gives
\begin{align}
    M_3 = \mathbb{E}[|Z_i |^3] & \lesssim 
    \alpha^3 \gamma \min\left\{1, \frac{L^2}{\sigma_p^2}\right\} 
    + \alpha^2 \beta \frac{L}{\sigma_q}
    \\
    & \quad + \alpha \beta^2\frac{\sigma_p^3 + \sigma_p^2 L + \sigma_p L^2}{\sigma_q}
    + (\beta \sigma_q)^3
    \\
    & \lesssim 
    \frac{p^3/\sigma_p^2}{q^3 /\sigma_q^2} \min\left\{1, \frac{L^2}{\sigma_p^2}\right\} 
    + \frac{p^2/ \sigma_p^2}{q^2 \sigma_q} L^2
    \\
    & \quad 
    + \frac{p/\sigma_p}{q \sigma_q^4} L^2 (\sigma_p^3 + \sigma_p^2 L + \sigma_p L^2) + \frac{L^3}{\sigma_q^3}.
\end{align}
 
\subsection{Toolbox}

The following properties of Gaussian distributions are useful for proving the lower bounds.

\begin{lemma} \label{lem:int_e}
\begin{align}
    &\int_{\mathbb{R}} e^{- \frac{(x-L)^2}{2b^2}} \frac{1}{\sqrt{2\pi} c} e^{- \frac{(x-L)^2}{2c^2}} dx = \frac{b}{\sqrt{b^2 + c^2}},
    \\
    &\int_{\mathbb{R}} e^{- \frac{(x+L)^2}{2b^2}} \frac{1}{\sqrt{2\pi} c} e^{- \frac{(x-L)^2}{2c^2}} dx = \frac{b}{\sqrt{b^2 + c^2}} e^{-\frac{2L^2}{b^2 + c^2}}.
\end{align}
\end{lemma}

\begin{lemma} \label{lem:int_ee}
\begin{align}
    &\int_{\mathbb{R}} e^{- \frac{(x-L)^2}{2a^2} - \frac{(x+L)^2}{2b^2}} \frac{1}{\sqrt{2\pi} c} e^{- \frac{(x-L)^2}{2c^2}} dx 
    \\
    = & \frac{ab}{\sqrt{a^2b^2 + a^2c^2 + b^2 c^2}} \exp\left( -2L^2 \frac{a^2 + c^2}{a^2b^2 + a^2c^2 + b^2 c^2}\right).
\end{align}
\end{lemma}

\begin{lemma} \label{lem:int_x_e}
\begin{align}
    &\int_{\mathbb{R}} x e^{- \frac{(x-L)^2}{2b^2}} \frac{1}{\sqrt{2\pi} c} e^{- \frac{(x-L)^2}{2c^2}} dx = \frac{b}{\sqrt{b^2 + c^2}}L,
    \\
    &\int_{\mathbb{R}} x e^{- \frac{(x+L)^2}{2b^2}} \frac{1}{\sqrt{2\pi} c} e^{- \frac{(x-L)^2}{2c^2}} dx = \frac{b}{\sqrt{b^2 + c^2} } e^{-\frac{2L^2}{b^2 + c^2}} \frac{b^2-c^2}{b^2 + c^2} L.
\end{align}
\end{lemma}

\begin{lemma} \label{lem:int_x2_e}
\begin{align}
    &\int x^2 e^{- \frac{x^2}{2b^2}} \frac{1}{\sqrt{2\pi} c} e^{- \frac{(x-L)^2}{2c^2}} dx 
    \\
    = & \frac{1}{\sqrt{2\pi} c} \frac{b^2 c^2}{2(b^2 + c^2)^{3/2}} \bigg( 
    \sqrt{2 \pi} bc \cdot \mathrm{erf}\left( \frac{x \sqrt{b^2 + c^2}}{\sqrt{2} bc}\right) 
    \\
    & - 2 x \sqrt{b^2 + c^2} \exp\left(-\frac{x^2}{2b^2} -\frac{x^2}{2c^2} \right)
    \bigg)  + \mathrm{constant}
    \\
    = & \frac{b^3 c^2}{2(b^2 + c^2)^{3/2}} \mathrm{erf}\left( \frac{x \sqrt{b^2 + c^2}}{\sqrt{2} bc}\right) 
    \\
    & - \frac{b^2c x}{\sqrt{2\pi} (b^2+c^2)^{3/2}} \exp\left(-\frac{x^2}{2b^2} -\frac{x^2}{2c^2} \right) + \mathrm{constant}.
\end{align}
\end{lemma}

\begin{lemma} \label{lem:int_x2_eplus}
\begin{align}
    &\int x^2 e^{- \frac{(x+M)^2}{2b^2}} \frac{1}{\sqrt{2\pi} c} e^{- \frac{(x-L)^2}{2c^2}} dx 
    \\
    = &  \frac{1}{\sqrt{2\pi} c}\frac{bc^2}{2(b^2 + c^2)^{5/2}} \exp\left( -\frac{(M+x)^2}{2b^2} -\frac{x^2}{2c^2}\right)
    \\
    &
    \bigg[ 
    \sqrt{2 \pi} c (b^4 + b^2 c^2 + c^2 M^2) \cdot \mathrm{erf}\left( \frac{b^2 x + c^2 M + c^2 x}{\sqrt{2} bc \sqrt{b^2 + c^2}}\right) \exp\left( \frac{(b^2x + c^2 M + c^2 x)^2}{2b^2 c^2 (b^2+c^2)}\right) 
    \\
    & - 2 b \sqrt{b^2 + c^2} (b^2 x + c^2(x-M))
    \bigg]  + \mathrm{constant}
    \\
    = &  \frac{b c^2 (b^4 + b^2 c^2 + c^2 M^2) }{2(b^2 + c^2)^{5/2}} \mathrm{erf}\left( \frac{b^2 x + c^2 M + c^2 x}{\sqrt{2} bc \sqrt{b^2 + c^2}}\right) \exp\left( -\frac{M^2}{2(b^2+c^2)}\right)
    \\
    & -
     \frac{b^2 c}{\sqrt{2 \pi}(b^2+c^2)^2}  \exp\left( -\frac{(M+x)^2}{2b^2} -\frac{x^2}{2c^2}\right) (b^2 x + c^2(x-M)) 
    + \mathrm{constant}.
\end{align}
\end{lemma}

\begin{lemma} \label{lem:erf}
For any $\epsilon>0$, $\mathrm{erf}(\epsilon) \le \frac{2}{\sqrt{\pi}} \epsilon.$
\end{lemma}

\begin{lemma} \label{lem:moment}
For any non-negative integer $p$,
\begin{align}
    & \int_{\mathbb{R}} |x|^p \frac{1}{\sqrt{2\pi} c} e^{- \frac{x^2}{2c^2}} dx =  c^p (p-1)!! \cdot 
    \begin{cases}
    \sqrt{\frac{2}{\pi}} & \textrm{if $p$ is odd}
    \\
    1 & \textrm{if $p$ is even}
    \end{cases}
\end{align}
\end{lemma}
\begin{lemma} \label{lem:int_abs}
\begin{align}
    &\int_{\mathbb{R}} \left|e^{- \frac{(x-L)^2}{2b^2}} - e^{- \frac{(x+L)^2}{2b^2}}\right| \frac{1}{\sqrt{2\pi} c} e^{- \frac{(x-L)^2}{2c^2}} dx
    \\
    = &  \frac{b}{\sqrt{b^2+c^2}} \mathrm{erf}\left(\sqrt{\frac{b^2 + c^2}{2 b^2 c^2}} L \right) +     \frac{b  e^{-\frac{2L^2}{b^2+c^2}}}{\sqrt{b^2+c^2}}\mathrm{erf}\left(\frac{c^2-b^2}{bc\sqrt{2 b^2 + 2c^2} } L \right)
    \\
    \le & \frac{4 }{\sqrt{\pi}} \frac{c^2}{c^2 + b^2} \frac{L}{c}.
\end{align}
\end{lemma}
\begin{lemma} \label{lem:int_xabs_2}
\begin{align}
    &\int_{\mathbb{R}} |x|\left( e^{- \frac{x^2}{2b^2}} - e^{- \frac{(x+M)^2}{2b^2}}\right)^2 \frac{1}{\sqrt{2\pi} c} e^{- \frac{x^2}{2c^2}} dx 
    \\
    = & \frac{bc^2 M }{(b^2 + 2c^2)^{3/2}} \exp\left( -\frac{b^2 + c^2}{2b^2 (b^2 + 2c^2)} M^2 \right) \mathrm{erf}\left( \frac{c M}{b \sqrt{2b^2 + 4c^2}}\right) 
    \\
    & + 
    \frac{bc^2 M }{(b^2 + 2c^2)^{3/2}} \exp\left( -\frac{1}{b^2 + 2c^2} M^2 \right) \mathrm{erf}\left( \frac{2 c M}{b \sqrt{2b^2 + 4c^2}}\right) 
    \\
    \le & \sqrt{\frac{8}{\pi}}  \frac{c^3 M}{(b^2 + 2c^2)^2} \left( \exp\left( -\frac{(b^2 + c^2) M^2}{2b^2 (b^2 + 2c^2)} \right) + \exp\left( -\frac{M^2}{b^2 + 2c^2} \right)\right)
    \\
    \le & \sqrt{\frac{32}{\pi}}  \frac{c^3 M}{(b^2 + 2c^2)^2}.
\end{align}
\end{lemma}
\begin{lemma} \label{lem:int_x2_abs}
\begin{align}
    &\int_{\mathbb{R}} x^2 \left| e^{- \frac{x^2}{2b^2}} - e^{- \frac{(x+M)^2}{2b^2}}\right| \frac{1}{\sqrt{2\pi} c} e^{- \frac{x^2}{2c^2}} dx 
    \\
    = & \frac{b^3c^2}{(b^2+c^2)^{3/2}} \mathrm{erf} \left( \frac{\sqrt{b^2 + c^2}}{\sqrt{8} bc}M\right)
    \\ 
    & + \frac{bc^2 (b^4 + b^2 c^2 + c^2M^2 )}{(b^2 + c^2)^{5/2}} \mathrm{erf}\left( \frac{(c^2-b^2)M}{2bc \sqrt{2b^2 + 2c^2}} \right) \exp\left(-\frac{M^2}{2b^2 + 2c^2}\right)
    \\
    \\
    & + \frac{2}{\sqrt{2\pi}} \frac{b^2 c^3 M}{ (b^2+c^2)^2} \exp\left(- \frac{M^2}{8b^2} - \frac{M^2}{8c^2} \right).
\end{align}
\end{lemma}

\begin{proof}
Let $L=M/2$. Since $\frac{x^2}{2b^2} \le \frac{(x+M)^2}{2b^2}$ when $x \ge -L$, and $\frac{x^2}{2b^2} > \frac{(x+M)^2}{2b^2}$ otherwise, we have
\begin{align}
    &\int_{\mathbb{R}} x^2 \left| e^{- \frac{x^2}{2b^2}} - e^{- \frac{(x+M)^2}{2b^2}}\right| \frac{1}{\sqrt{2\pi} c} e^{- \frac{x^2}{2c^2}} dx 
    \\
    = & \left(\int_{-L}^{+\infty} - \int_{-\infty}^{-L} \right) x^2 \left( e^{- \frac{x^2}{2b^2}} - e^{- \frac{(x+M)^2}{2b^2}}\right) \frac{1}{\sqrt{2\pi} c} e^{- \frac{x^2}{2c^2}} dx.
\end{align}
By Lemma~\ref{lem:int_x2_e},
\begin{align}
    & \left(\int_{-L}^{+\infty} - \int_{-\infty}^{-L} \right) x^2 e^{- \frac{x^2}{2b^2}}  \frac{1}{\sqrt{2\pi} c} e^{- \frac{x^2}{2c^2}} dx
    \\
    & = \frac{b^3c^2}{(b^2+c^2)^{3/2}} \mathrm{erf} \left( \frac{\sqrt{b^2 + c^2}}{\sqrt{2} bc}L\right)
    - \frac{2b^2c L}{\sqrt{2\pi}(b^2+c^2)} \exp\left( -\frac{L^2}{2b^2} - \frac{L^2}{2c^2}  \right).
\end{align}
By Lemma~\ref{lem:int_x2_eplus},
\begin{align}
    & \left(\int_{-L}^{+\infty} - \int_{-\infty}^{-L} \right) x^2 e^{- \frac{(x+M)^2}{2b^2}}  \frac{1}{\sqrt{2\pi} c} e^{- \frac{x^2}{2c^2}} dx
    \\
    = &  - \frac{b c^2 (b^4 + b^2 c^2 + c^2 M^2) }{(b^2 + c^2)^{5/2}} \mathrm{erf}\left( \frac{c^2-b^2}{\sqrt{2} bc \sqrt{b^2 + c^2}} L\right) \exp\left( -\frac{4L^2}{2(b^2+c^2)}\right)
    \\
    & -
     \frac{2 b^2 c}{\sqrt{2 \pi}(b^2+c^2)^2}  \exp\left( -\frac{L^2}{2b^2} -\frac{L^2}{2c^2}\right) (b^2 + 3c^2) L.
\end{align}
Combining the terms completes the proof.
\end{proof}

\end{document}